\pdfoutput=1
\documentclass[11pt,a4paper,table]{article}
\usepackage[hyperref]{aacl-ijcnlp2020}
\usepackage{times}
\usepackage{latexsym}

\usepackage{svg}
\usepackage{subfiles}
\usepackage{url}
\usepackage{afterpage}
\usepackage{graphicx}
\usepackage{xcolor}
\usepackage{amsfonts}
\usepackage{amsmath}
\allowdisplaybreaks
\usepackage{amsthm}
\usepackage{amssymb}
\usepackage{mathtools}
\usepackage{bbm}

\usepackage{xspace}
\usepackage{xargs}
\usepackage[colorinlistoftodos,prependcaption,textsize=tiny]{todonotes}
\usepackage{marginnote}
\usepackage{multirow}
\usepackage{adjustbox}
\usepackage{booktabs}
\usepackage{commath}
\usepackage{bbm}
\usepackage{wrapfig}
\usepackage[normalem]{ulem}
\usepackage{color, colortbl}

\theoremstyle{definition}
\usepackage{amsthm}
\theoremstyle{definition}
\newtheorem{definition}{Definition}[section]

\newtheorem{proposition}{Proposition}
\newtheorem{lemma}{Lemma}[section]
\newtheorem{property}{Property}

\definecolor{lightblue}{rgb}{0.93,0.95,1.0}
\newcommandx{\cmtn}[2][1=]{\todo[linecolor=blue,backgroundcolor=blue!10,bordercolor=blue,#1]{MN: #2}\xspace}
\newcommandx{\chtx}[2][1=]{\todo[linecolor=blue,backgroundcolor=blue!10,bordercolor=green,#1]{HTX: #2}\xspace}
\usepackage{microtype}
\aclfinalcopy

\usepackage{amsmath,amsfonts,bm}

\def\eqref#1{equation~\ref{#1}}

\def\1{\bm{1}}

\def\vp{{\bm{p}}}

\DeclareMathAlphabet{\mathsfit}{\encodingdefault}{\sfdefault}{m}{sl}
\SetMathAlphabet{\mathsfit}{bold}{\encodingdefault}{\sfdefault}{bx}{n}

\title{A Systematic Characterization of Sampling Algorithms\\ for Open-ended Language Generation}

\author{Moin Nadeem\thanks{~~Equal contribution.}  \\
Massachusetts Institute of Technology \\
\texttt{mnadeem@mit.edu}
\And 
Tianxing He\footnotemark[1] \\
Massachusetts Institute of Technology \\
\texttt{cloudygoose@csail.mit.edu}
\AND
Kyunghyun Cho \\
New York University \\
\texttt{kyunghyun.cho@nyu.edu} \\
\And 
James Glass \\
Massachusetts Institute of Technology \\
\texttt{glass@mit.edu}
}

\date{}

\begin{document}
\maketitle
\begin{abstract}
This work studies the widely adopted ancestral sampling algorithms for auto-regressive language models, which is not widely studied in the literature. We use the quality-diversity (Q-D) trade-off to investigate three popular sampling algorithms (top-$k$, nucleus and tempered sampling). We focus on the task of open-ended language generation. We first show that the existing sampling algorithms have similar performance. 
After carefully inspecting the transformations defined by different sampling algorithms, we identify three key properties that are shared among them: \textit{entropy reduction}, \textit{order preservation}, and \textit{slope preservation}. To validate the importance of the identified properties, we design two sets of new sampling algorithms: one set in which each algorithm satisfies all three properties, and one set in which each algorithm violates at least one of the properties. We compare their performance with existing sampling algorithms, and find that violating the identified properties could lead to drastic performance degradation, as measured by the Q-D trade-off. On the other hand, we find that the set of sampling algorithms that satisfies these properties performs on par with the existing sampling algorithms.\footnote{Our data and code are available at \url{https://github.com/moinnadeem/characterizing-sampling-algorithms}.}
\end{abstract}

\section{Introduction}
\label{sec:intro}

A language model (LM) is a central module for natural language generation (NLG) tasks \citep{trends-nlp} such as machine translation \citep{wu17ganmt}, dialogue response generation \citep{dialogue17jiwei}, image captioning \citep{coco14tsung}, and related tasks. Given a trained LM, finding the best way to generate a sample from it has been an important challenge for NLG applications.

Decoding, i.e., finding the most probable output sequence from a trained model, is a natural principle for generation. The beam-search decoding algorithm approximately finds the most likely sequence by performing breadth-first search over a restricted search space. It has achieved success in machine translation, summarization, image captioning, and other subfields. 

However, in the task of open-ended language generation (which is the focus of this work), a significant degree of \textit{diversity} is required. For example, conditioned on the prompt ``\texttt{The news says that ...}'', the LM is expected to be able to generate a wide range of interesting continuations. While the deterministic behavior of decoding algorithms could give high-quality samples, they suffer from a serious lack of diversity.

\begin{figure}
    \centering
    \includegraphics[width=\columnwidth]{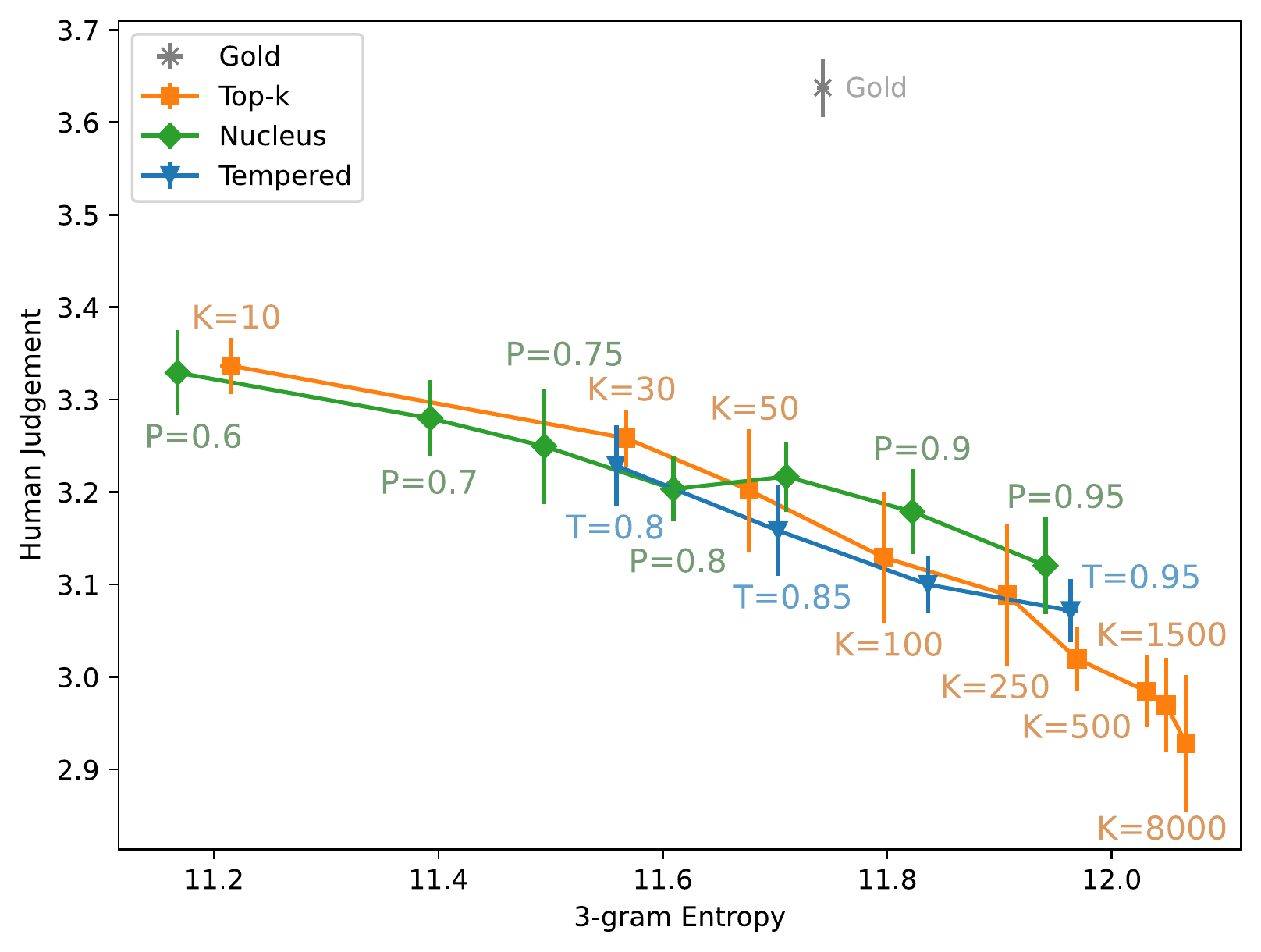}
    \caption{Human evaluation (y-axis: quality, x-axis: diversity, both are the bigger the better) shows that the generation performance of existing sampling algorithms are on par with each other.}
    \label{fig:giga_humaneval_existing}
\end{figure}

This need for diversity gives rise to a wide adoption of various sampling algorithms. Notably, top-$k$ sampling \citep{fan-etal-2018-hierarchical}, nucleus sampling \citep{Holtzman2020The}, and tempered sampling \citep{shortgan18massimo} have been used in open-ended generation \citep{radford18gpt2, shortgan18massimo}, story generation \citep{fan-etal-2018-hierarchical}, and dialogue response generation \citep{zhang2019dialogpt}.  However, the sampling algorithm and the hyperparameter are usually chosen via heuristics, and a comprehensive comparison between existing sampling algorithm is lacking in the literature. More importantly, \textbf{the underlying reasons behind the success of the existing sampling algorithms still remains poorly understood}.

In this work, we begin by using the quality-diversity (Q-D) trade-off \citep{shortgan18massimo} to compare the three existing sampling algorithms. For automatic metrics, we use the BLEU score for quality and n-gram entropy for diversity. We also correlate these automatic metrics with human judgements. The first observation we draw is that top-$k$ , nucleus and tempered sampling perform on par in the Q-D trade-off, as shown in Figure \ref{fig:giga_humaneval_existing}.
Motivated by this result, we extract three key properties by inspecting the transformations defined by the sampling algorithms: (1) \textit{entropy reduction}, (2) \textit{order preservation} and (3) \textit{slope preservation}. We prove all three properties hold for the three existing sampling algorithms. 

We then set out to systematically validate the importance of the identified properties. To do so, we design two sets of new sampling algorithms in which each algorithm either violates one of the identified properties, or satisfies all properties. Using the Q-D trade-off, we compare their efficacy against existing algorithms, and find that violating these identified properties could result in significant performance degradation. More interestingly, we find that the set of sampling algorithms that satisfies these properties has generation performance that matches the performance of existing sampling algorithms.

\section{Sampling Algorithms for \\ Autoregressive Language Models}
\label{sec:sampling_big}

\subsection{Autoregressive Language Modeling}
The task of autoregressive language modeling is to learn the probability distribution of the $(l+1)$-th word $W_{l+1}$ in a sentence $W$ conditioned on the word history $W_{1:l} := (W_1, \dots, W_{l})$ and context $C$. Here, we use $W_i \in V$ to denote a discrete random variable distributed across a fixed vocabulary $V$. In this work, the vocabulary is constructed on sub-word level \citep{bpe-sennrich-etal-2016-neural}. 

Given a training set $D$, maximum likelihood estimation (MLE) has been the most popular framework to train an autoregressive LM \citep{tomas10rnn}. MLE training minimizes the negative log-likelihood (NLL) objective below:
\begin{align}
\small
\label{eq:mleobj}
\begin{split}
 L_{\text{MLE}} = \frac{1}{|D|} \sum_{(W,C) \in D}-\Sigma_{l=0}^{L-1}\log P_{\theta}(W_{l+1}|W_{1:l}, C),
\end{split}
\end{align}
where $\theta$ denotes model parameters, and $P_{\theta}(\cdot\ |\ W_{1:l})$ denotes the conditional model distribution of $W_{l+1}$ given a prefix $W_{1:l}$. For simplicity, we assume all sentences are of length $L$ in the formulations. Since this work focuses on sampling from a given model instead of training it, in the rest of the paper, we abbreviate $P_\theta(\cdot)$ as $P(\cdot)$ for brevity. 

\subsection{Existing Sampling Algorithms}
\label{sec:sampling_exsiting algorithms}

Given a trained LM and a context $C$, an ancestral sampling algorithm seeks to generate a sequence from $P(W|C)$ by sampling token-by-token from a transformed version of $P(W_{l+1}|W_{1..l},C)$. We now review and formulate three popular sampling algorithms: top-$k$ \citep{fan-etal-2018-hierarchical}, nucleus \citep{Holtzman2020The}, and tempered \citep{ackley1988boltzmann, shortgan18massimo} sampling. 

We view these algorithms as different transformations applied to the distribution $P(W_{l+1}|W_{1..l},C)$. First, we treat the conditional distribution $P(W_{l+1}|W_{1..l},C)$ as a \textit{sorted} vector $\vp$ of length $|V|$. By sorting, we rearrange the elements such that if $i<j \rightarrow p_i>=p_j$.\footnote{The token indexes are also permutated accordingly.} We list the transformations and their intuition below:

\begin{definition}
\label{def_topk}
(\textbf{Top-$k$})  
In top-$k$ sampling, we only sample from the top $K$ tokens:
\begin{equation}
    \hat{p}_i=\frac{p_i \cdot \mathbbm{1}\{i \leq K \}}{\sum^{K}_{j=1}p_j},
\end{equation}
where $\mathbbm{1}$ is the indicator function, and $K$ ($1\leq K \leq |V|$) is the hyperparameter.
\end{definition}

\begin{definition}
\label{def_nucleus}
(\textbf{Nucleus})  
With a hyperparameter $P$ ($0<P\leq 1$), in nucleus sampling, we sample from the top-$P$ mass of $\vp$:
\begin{equation}
\begin{split}
\hat{p}_i &= \frac{p'_i}{\sum^{|V|}_{j=1} p'_j},
\end{split}
\end{equation}
where $p'_i=p_i \cdot \mathbbm{1}\{\sum_{j=1}^{i-1}p_j < P\}$.
\end{definition}

\begin{definition}
\label{def_temp}
(\textbf{Tempered})  
In tempered sampling, the log probabilities are scaled by $\frac{1}{T}$:
\begin{equation}
    \hat{p}_i= \frac{\exp(\log(p_i)/T)}{\sum^{|V|}_{j=1} \exp(\log(p_j)/T)}.
\end{equation}
In this work, we assume $0<T<1$, i.e., the distribution is only made sharper\footnote{One could also use $T>1$, but it does not work well in practice.}.
\end{definition}

We additionally experiment with a combined version of top-$k$ and tempered sampling:
\begin{definition}
\label{def_topk_temp}
(\textbf{Tempered Top-$k$})
We combine the transformation defined by top-$k$ and tempered sampling:
\begin{equation}
\begin{split}
\hat{p}_i = \frac{p'_i}{\sum^{|V|}_{j=1} p'_j},
\end{split}
\end{equation}
where $p'_i = \exp(\log(p_i)/T) \cdot \mathbbm{1}\{i \leq K \} $.
We set $1\leq K \leq |V|$ and $0<T<1$.
\end{definition}

Throughout this work we use $\hat{\vp}$ to denote the normalized version of the transformed distribution.
All algorithms have hyperparameters to control the entropy of the transformed distribution. For example, $K$ in top-$k$ sampling controls the size of the support of the resulting distribution. We will formalize this statement in Property \ref{p_entropy} below.

\section{Properties of Sampling Algorithms}
As we will show in Section \ref{sec:result_exsit_exp} (also Figure \ref{fig:giga_humaneval_existing}), top-$k$, nucleus and tempered sampling perform on par with each other under our evaluation. This key observation makes us question: \textit{What are the core principles underlying the different algorithms that lead to their similar performance?}

To answer this question, in this section, we identify three core properties that are provably shared by the existing sampling algorithms. We then design experiments to validate their importance.

\subsection{Identifying Core Properties}
\label{sec:main_core_property}

By inspecting the transformations listed in Definition \ref{def_topk}, \ref{def_nucleus} and \ref{def_temp}, we extract the following three properties:

\begin{property}
\label{p_entropy}
\textbf{(Entropy Reduction)}: The transformation strictly decrease the entropy of the distribution. Formally, $\mathcal{H}(\hat{\vp})<\mathcal{H}(\vp)$, where $\mathcal{H}(\vp)=-\sum^{|V|}_{i=1} p_i \log p_i$.
\end{property}

\begin{property}
\label{p_order}
\textbf{(Order Preservation)}: The order of the elements in the distribution is preserved. Formally, $ p_i \geq p_j \rightarrow \hat{p}_i \geq \hat{p}_j $.
\end{property}

\begin{property}
\label{p_slope}
\textbf{(Slope Preservation)}: The ``slope'' of the distribution is preserved. Formally, $\forall \hat{p}_i > \hat{p}_j > \hat{p}_k > 0$ (i.e., they are not truncated), we have $\frac{\log p_i-\log p_j}{\log p_j-\log p_k}=\frac{\log\hat{p}_i-\log\hat{p}_j}{\log\hat{p}_j-\log\hat{p}_k}$.
\end{property}

The order preservation property implies that truncation can only happen in the tail of the distribution, which aligns with top-$k$ and nucleus sampling. The slope preservation property is stronger than the order preservation property in that not only the ordering, but also the relative magnitude of the elements in the distribution needs to be somewhat preserved by the transformation. 

All these three properties are shared by the three existing sampling algorithms:
\begin{proposition}
\label{prop_propertyholds}
Property \ref{p_entropy}, \ref{p_order} and \ref{p_slope} hold for the top-$k$, nucleus and tempered sampling transformations formulated in Definitions \ref{def_topk}, \ref{def_nucleus} and \ref{def_temp}.
\end{proposition}
\begin{proof}
See Appendix \ref{app_proof}.
\end{proof}

We then set out to validate the importance of these identified properties in the aspects of \textit{necessity} and \textit{sufficiency}. To do so, we design two sets of new sampling algorithms in which each algorithm either violates one of the identified properties, or satisfies all properties. We list them in the next section.

\subsection{Designed Sampling Algorithms}
\label{sec:sampling_deisnged}

\paragraph{Property-violating algorithms}
To validate the necessity of each property, we design several sampling algorithms which \textit{violate at least one of the identified properties}. In our experiments, we check whether that violation leads to a significant degradation in performance. We list them below:

\begin{definition}
\label{def_targetentropy}
\textbf{(Target Entropy)} 
Based on tempered sampling, target entropy sampling tunes the temperature $t$ such that the transformed distribution has entropy value equal to the hyperparameter $E$ ($0 < E \leq \log|V| $). We formulate it below:
\begin{equation}
    \hat{p}_i= \frac{\exp(\log(p_i)/t)}{\sum^{|V|}_{j=1} \exp(\log(p_j)/t)},
\end{equation}
where $t$ is selected such that $H(\hat{\vp})=E$.
\end{definition}
Target entropy sampling violates entropy reduction, because when $H(\vp) < E$, the entropy will be tuned up (i.e., $H(\hat{\vp})>H(\vp)$).

\begin{definition}
\textbf{(Random Mask)}
In random mask sampling, we randomly mask out tokens in the distribution with rate $R$. We formluate it below:
\begin{equation}
\begin{split}
 \hat{p}_i &= \frac{p'_i}{\sum^{|V|}_{j=1}p'_j},
\end{split}
\end{equation}
where $p'_i = p_i \cdot \mathbbm{1}\{i = 1~\text{or}~u_i > R \}$ and $u_i \sim U(0,1)$. The hyperparameter $R$ ($0 < R \leq 1$) controls the size of the support of the resulting distribution. In Appendix \ref{app_auxiliaryplots}, we show it is crucial that the token which is assigned the largest probability ($p_1$) is never be masked. 
\end{definition}

Random mask sampling is different from top-$k$ or nucleus sampling in that the masking not only happens in the tail of the distribution. Therefore, it violates the order preservation property.

\begin{definition}
(\textbf{Noised Top-$k$}) We add a \textit{sorted} noise distribution to the result from top-$K$ transformation, and the weight of the noise distribution is controlled by a hyperparameter $W$ ($0 \leq W \leq 1$). We formulate it below:
\begin{equation}
\hat{\vp} = (1-W) \hat{\vp}^\text{top-K} + W \vp^\text{noise-K},
\end{equation}
where $\vp^\text{noise-K}$ is a uniformly sampled \textit{sorted K-simplex}, which satisfies $\sum^K_{i=1} p_i^\text{noise-K}=1$ and $i<j \rightarrow p_i^\text{noise-K} \geq p_j^\text{noise-K} \geq 0$.
\end{definition}
The sorted nature of the noise distribution $\vp^\text{noise-K}$ maintains order preservation. However, it violates slope preservation, and the noise weight $W$ controls the degree of the violation.

\paragraph{Property-satisfying algorithms}
To validate the sufficiency of the identified properties, we design two new sampling algorithms for which \textit{all three properties hold}. And in our experiments we check whether their performance is on par with the existing sampling algorithms. We list them below:

\begin{definition}
(\textbf{Random Top-$k$})  
We design a randomized version of top-$k$ sampling: At each time step, we sample a uniformly random float number $u \sim U(0,1)$, and use it to specify a top-$k$ truncation:
\begin{equation}
\begin{split}
        \hat{p}_i=\frac{p_i \cdot \mathbbm{1}\{i \leq k \}}{\sum^{k}_{j=1}p_j},
\end{split}
\end{equation}
where $k=\left \lfloor 1+M \cdot u \right \rfloor$.
 The hyperparameter $M$ ($1 \leq M < |V|$) controls the maximum truncation threshold.
\end{definition}

\begin{definition}
(\textbf{Max Entropy})  
Max entropy sampling is similar to target entropy sampling (Definition \ref{def_targetentropy}). However to match entropy reduction (Property \ref{p_entropy}), we only tune the temperature when $\mathcal{H}(\vp) > E$, where $E$ is the hyperparameter ($0 < E \leq \log|V| $):
\begin{equation}
    \hat{p}_i=\begin{cases} 
        \frac{\exp(\log(p_i)/t)}{\sum^{|V|}_{j=1}\exp(\log(p_j)/t)}, & \text{if}~\mathcal{H}(\vp) > E \\
        p_i, & \text{otherwise}
        \end{cases}, \\
\end{equation}
where $t$ is selected so that $\mathcal{H}(\hat{\vp})=E$.
\end{definition}

It is easy to prove that Property \ref{p_entropy}, \ref{p_order}, and \ref{p_slope} holds for the transformations defined by random top-$k$ and max entropy sampling, and we omit the proof for brevity.

\section{Experiment Setup}
In this section, we first establish evaluation protocols, and then describe the model and data we use for the open-ended language generation task.

\subsection{Evaluation via the Q-D Trade-off}
\label{sec:setup_evaluation}
How to efficiently measure the generation performance of a NLG model has been an important open question. Most existing metrics either measure the \textit{quality} aspect (e.g. BLEU score) or the \textit{diversity} (e.g. n-gram entropy) aspect. To make the situation more complicated, each sampling algorithm has its own hyperparameters which controls the trade-off between quality and diversity.

To address the challenges above, we adopt the quality-diversity trade-off proposed by \citet{shortgan18massimo}. In the Q-D trade-off, we perform a fine-grained sweep of hyperparameters for each sampling algorithm, and compute the quality and diversity score for each configuration. We report two pairs of Q/D metrics, with one pair using automatic evaluation and the other using human evaluation. In the next two sections, we describe the metrics we use, and refer readers to \citet{shortgan18massimo} for more intuition behind the Q-D trade-off.

\subsubsection{Automatic Evaluation}
\label{sec:automatic_eval_bleu}
For automatic metrics, we adopt the corpus-BLEU \citep{yu2016seqgan} metric to measure quality and the self-BLEU \citep{zhu2018texygen} metric to measure diversity. We formulate them below.

Given a batch of generated sentences $S_\text{gen}$ and a batch of sentences from ground-truth data as references $S_\text{ref}$, corpus-BLEU returns the average BLEU score \citep{papineni-etal-2002-bleu} of every model generated sentence against the reference set:
\begin{equation}
\small
\begin{split}
    \text{corpus-BLEU}(S_\text{gen}, S_\text{ref}) = \frac{1}{|S_\text{gen}|} \sum_{W\in S_\text{gen}} \text{BLEU} & (W, S_\text{ref}).
\end{split}
\end{equation}
A higher corpus-BLEU score means that the generated sequences has better quality in that it has higher ngram-level overlap with the reference data. Based on the same intuition, we define the self-BLEU metric to quantify the diversity aspect:
\begin{equation}
\small
    \text{self-BLEU}(S_\text{gen}) = \text{corpus-BLEU}(S_\text{gen}, S_\text{gen}),
\end{equation}
where a lower self-BLEU score means that the samples have better diversity.

In our experiments, we feed the first ten subwords of every sample from test set to the model, and compare the model-generated sequences to the reference samples in the validation set. We use 10,000 samples to compute corpus-BLEU or self-BLEU, i.e., $|S_\text{gen}|=|S_\text{ref}|=10,000$.

Automatic evaluation enables us to do a fine-grained sweep of the hyperparameters for each sampling algorithm, and compare them in the quality-diversity trade-off. However, observations from automatic evaluation could be misaligned with human evaluation \citep{belz-reiter-2006-comparing}. Therefore, we confirm our key observations with human evaluation. 

\subsubsection{Human Evaluation}
\paragraph{Quality} We ask a pool of 602 crowdworkers on Amazon Mechanical Turk to evaluate various sampling configurations in the quality aspect. Each worker is presented a set of ten samples along with the prompts (prefixes). They are then asked to rate how likely the sentence would appear in a news article between 0 and 5 (Invalid, Confusing, Unspecific, Average, Expected, and Very Expected respectively). 

We focus on the Gigaword dataset for human evaluation since news articles are ubiquitous and do not often require expert knowledge for quality judgement. For each configuration (sampling algorithm and hyperparameter pair) we ask crowdworkers to rate 200 samples in total. To get an accurate rating for each sample, we enlist 25 different crowdworkers to rate each sample. We report mean and standard deviation from 5 independent runs (each with 40 samples) as error bar.

By manual inspection, we find that the time spent in the annotations is a good indicator of the quality of the rating. Therefore, we estimate the human judgement score for a sample as the average rating of the 20 crowdworkers (out of 25) who took the most time to rate the samples. We provide further details about our setup in Appendix \ref{appendix:mturk_details} and \ref{appendix:convergence}.

\paragraph{Diversity}
It is difficult for human annotators to estimate diversity of text \cite{huse19tatsunori}. Therefore, we use the \textit{n-gram entropy} metric \citep{ yizhe18aim,negtrain19he} . Given $S_\text{gen}$ which contains a large number of samples, we measure its diversity using the following formulation:
\begin{equation}
    \mathcal{H}^\text{$n$-gram}(S_\text{gen}) = \sum_{g \in G_n} - r(g) \log r(g),
\end{equation}
where $G_n$ is the set of all n-grams that appeared in $S_\text{gen}$, and $r(g)$ refers to the ratio (frequency) of n-gram $g$ w.r.t. all n-grams in the $S_\text{gen}$. For the estimation of n-gram entropy, we generate 50,000 samples from each sampling configuration.

We will report human quality score either paired with n-gram entropy or with self-BLEU as diversity metric. We find they give similar observations. %

\subsection{Model and Datasets}
\label{sec:setup_model_data}
We separately fine-tune GPT2-small \cite{radford18gpt2, Wolf2019HuggingFacesTS} (110M parameters) on the Gigaword \citep{graff2003english, napoles-gigaword} and the Wikitext-103 \citep{DBLP:conf/iclr/MerityX0S17} datasets. We use the same tokenization as GPT-2, and add additional padding and end-of-sequence tokens (\texttt{[EOS]}) to the sentences.

To generate a sequence, we feed a length-10 prefix from test data into the fine-tuned GPT-2 model, and use a sampling algorithm to complete the sentence. Since shorter samples are more difficult to judge in quality \citep{ippolito-etal-2020-automatic}, we filter all generated sentence completions to be between 40 and 50 subwords, and filter our validation and test set to meet the same requirements. %
To permit validation and test sets that are large enough to prefix 10,000 sentences for the corpus-BLEU metric, we re-chunk the first 80\% of the Gigaword dataset for the training set, 15\% for validation, and the last 5\% for the test set. %
Similarly, we re-chunk the first 97\% of the Wikitext-103 dataset for training, and leave 1.5\% for validation and 1.5\% for test.

\section{Empirical Results}
First, we compare existing sampling algorithms, and then move on to validate the necessity and sufficiency of the identified properties.

\subsection{Comparison of Existing Algorithms}
\label{sec:result_exsit_exp}

\begin{figure}
    \centering
    \includegraphics[width=\columnwidth]{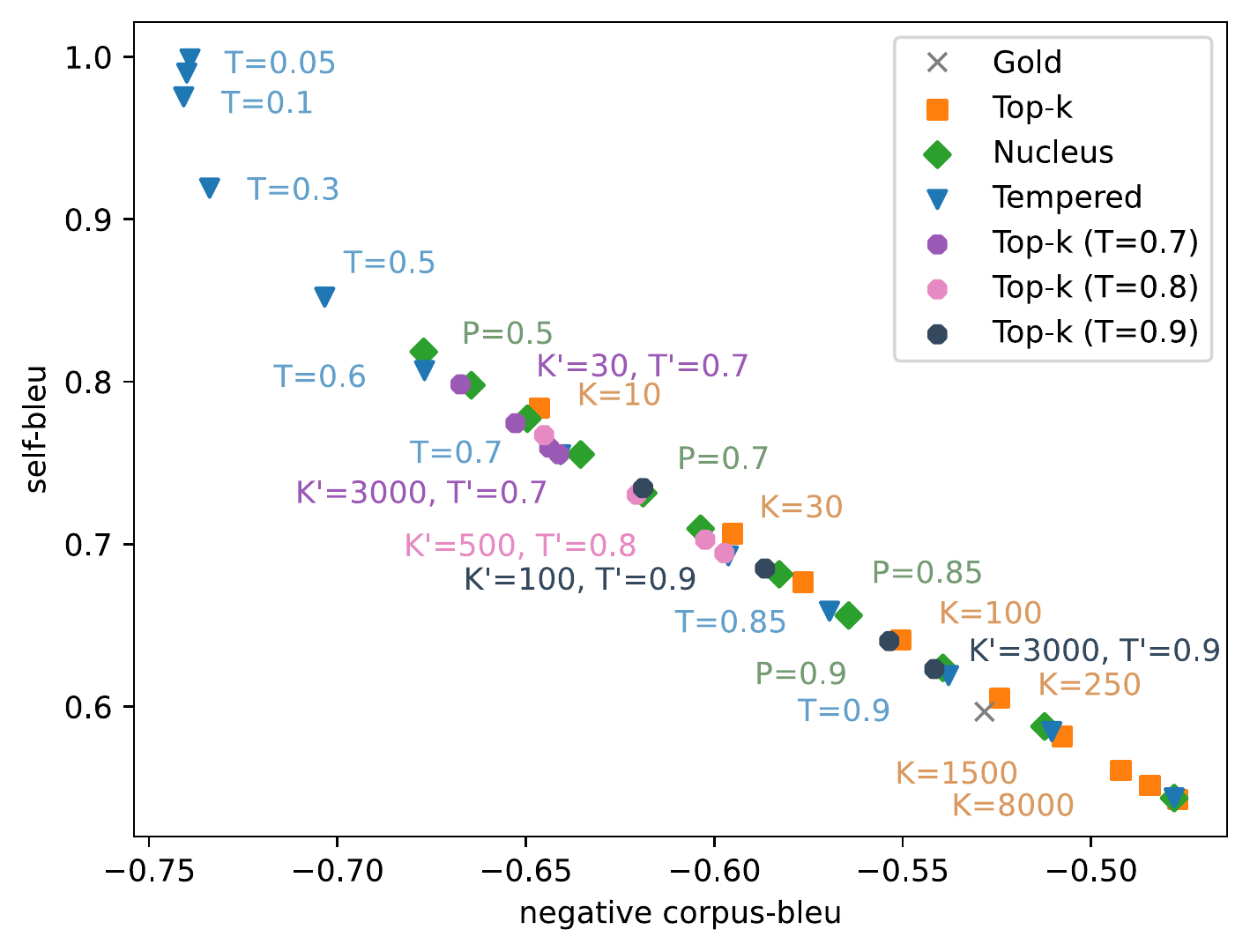}
    \caption{The performance (x-axis: quality, y-axis: diversity, both are the smaller the better) of top-$k$, nucleus, tempered and tempered top-$k$ sampling are on par on the Gigaword dataset, as shown by automatic evaluation.}
    \label{fig:giga_bleu_existing}
\end{figure}

We compare top-$k$, nucleus, and tempered sampling via automatic and human evaluation. We do a fine-grained sweep of hyperparameters for each sampling algorithm on the Gigaword dataset. The results are shown in Figure \ref{fig:giga_humaneval_existing} (human evaluation) and Figure \ref{fig:giga_bleu_existing} (automatic evaluation). 
We also show the quality and diversity score for human text in the test data for reference, which is labeled as gold.

Both automatic and human evaluations demonstrate that the performance of top-$k$, nucleus and tempered sampling are on par with each other, with no significant gap. When the hyperparameters ($K$, $P$ and $T$) are tuned so that different sampling has the same diversity (measured by self-BLEU or n-gram entropy), their quality (measured by corpus-BLEU or human rating) are close. 

Additionally, we compare tempered top-$k$ sampling with the existing algorithm also in Figure \ref{fig:giga_bleu_existing}. We find that adding the tempered transformation only moves top-$k$ sampling along the Q-D trade-off, instead of yielding a better or a worse sampling algorithm. For example, the performance of the $K=500, T=0.8$ configuration for tempered top-$k$ sampling is very close to the $K=30$ configuration for the top-$k$ sampling. 

\begin{figure}
    \centering
    \includegraphics[width=\columnwidth]{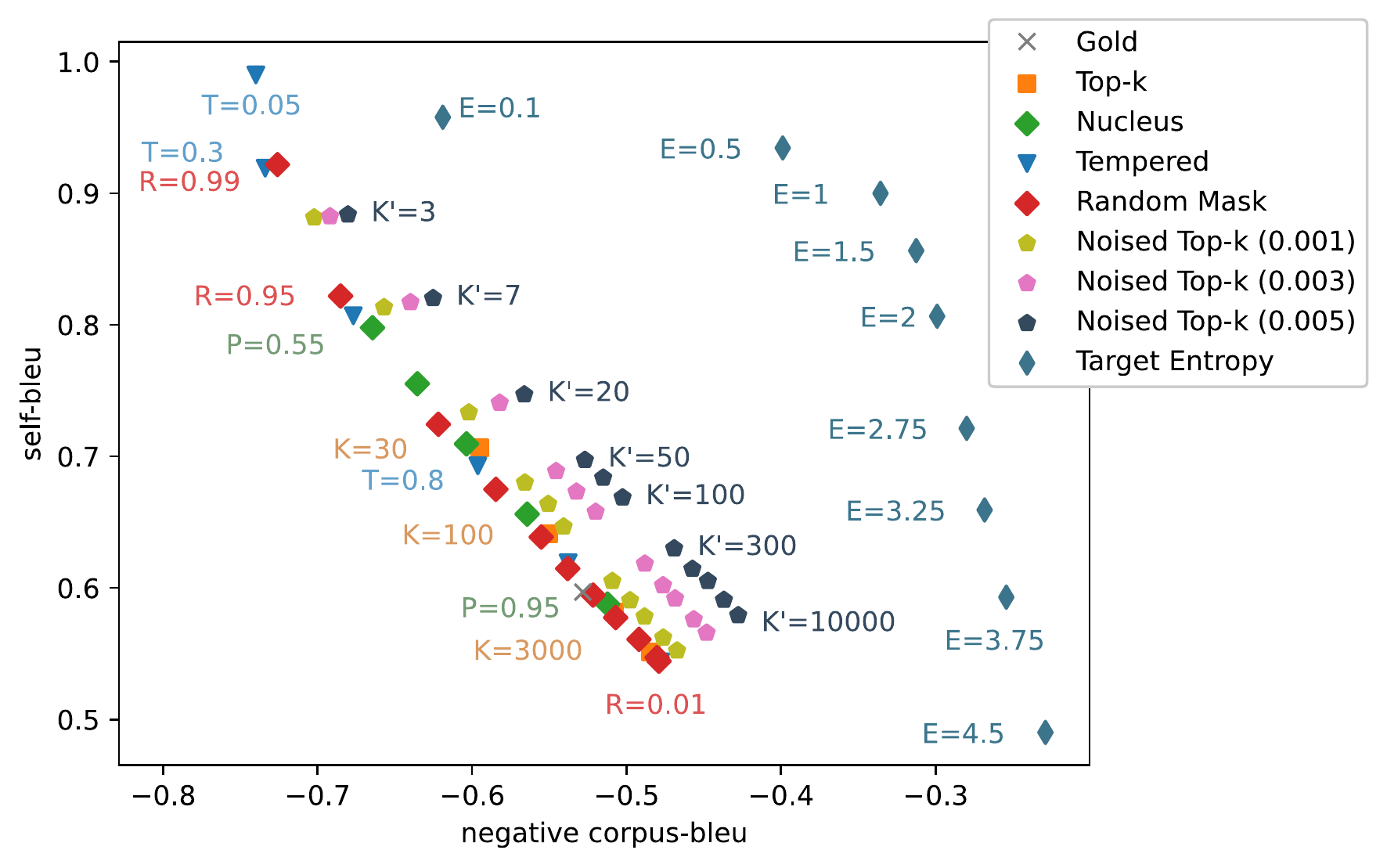}
    \caption{Automatic evaluation of the noised top-$k$, target entropy, and random mask sampling proposed to validate the necessity of the identified properties. The results show that violation of entropy reduction and slope preservation could lead to drastic performance degradation, while the order preservation property could be further relaxed.}
    \label{fig:giga_sufficiency_bleu}
\end{figure}

Motivated by these observations, we identify three core properties (elaborated in Section \ref{sec:main_core_property}) that are shared among the sampling algorithms: \textit{entropy reduction}, \textit{order preservation} and \textit{slope preservation}. In the following two sections, we present experiments validating the necessity or sufficiency aspect of the properties.

\subsection{Property-violating Algorithms}

In Figure \ref{fig:giga_sufficiency_bleu}, we compare the generation performance of the  property-violating sampling algorithms (designed in Section \ref{sec:sampling_deisnged}), against the existing algorithms using automatic evaluation on the Gigaword dataset. We make the following observations: First, the target entropy sampling, which violates entropy reduction, has significantly worse performance; Second, even with small noise weight $W$, the performance of noised top-$k$ sampling degrades from the original top-$k$ sampling, and the gap becomes larger as $W$ increases; Last, the random mask sampling is on par with the existing sampling algorithms in performance. We further confirm this observation with human evaluation in Figure \ref{fig:giga_humaneval_newsampling}. 

\begin{figure}[t]
    \centering
    \includegraphics[width=\columnwidth]{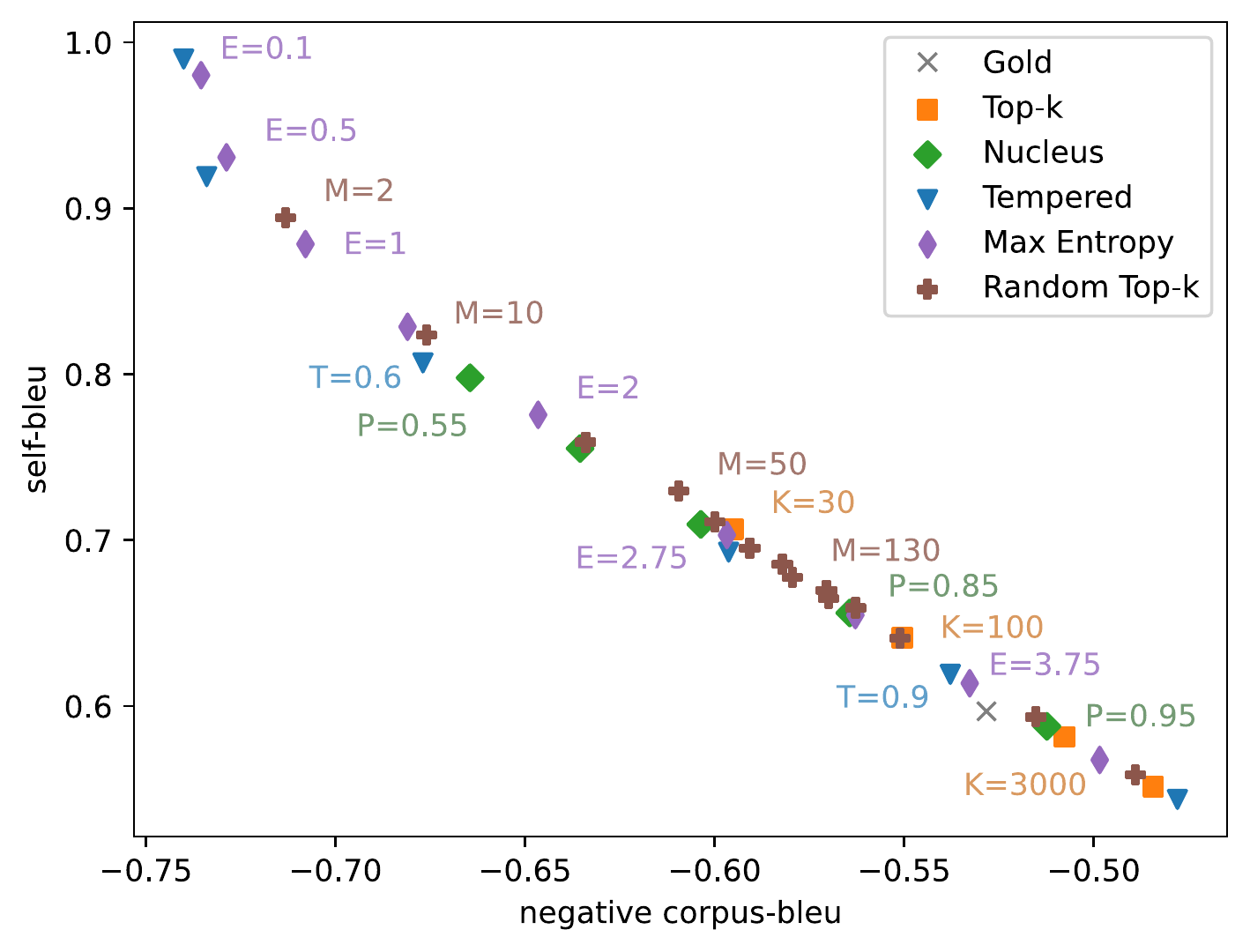}
    \caption{The proposed random top-$k$ and max entropy schedulers, which meet the identified properties, are on par in performance with existing methods in automatic evaluation on the Gigaword dataset.}
    \label{fig:giga_humaneval_newsampling}
\end{figure}

These results suggest that the violation of entropy reduction or slope preservation could lead to drastic performance degradation. On the other hand, the competitive performance of random mask sampling suggests that order preservation could be further relaxed.

In the next section, we investigate the sufficiency aspect of the identified properties.

\subsection{Property-satisfying Algorithms}

\begin{figure}[h]
    \centering
    \includegraphics[width=\columnwidth]{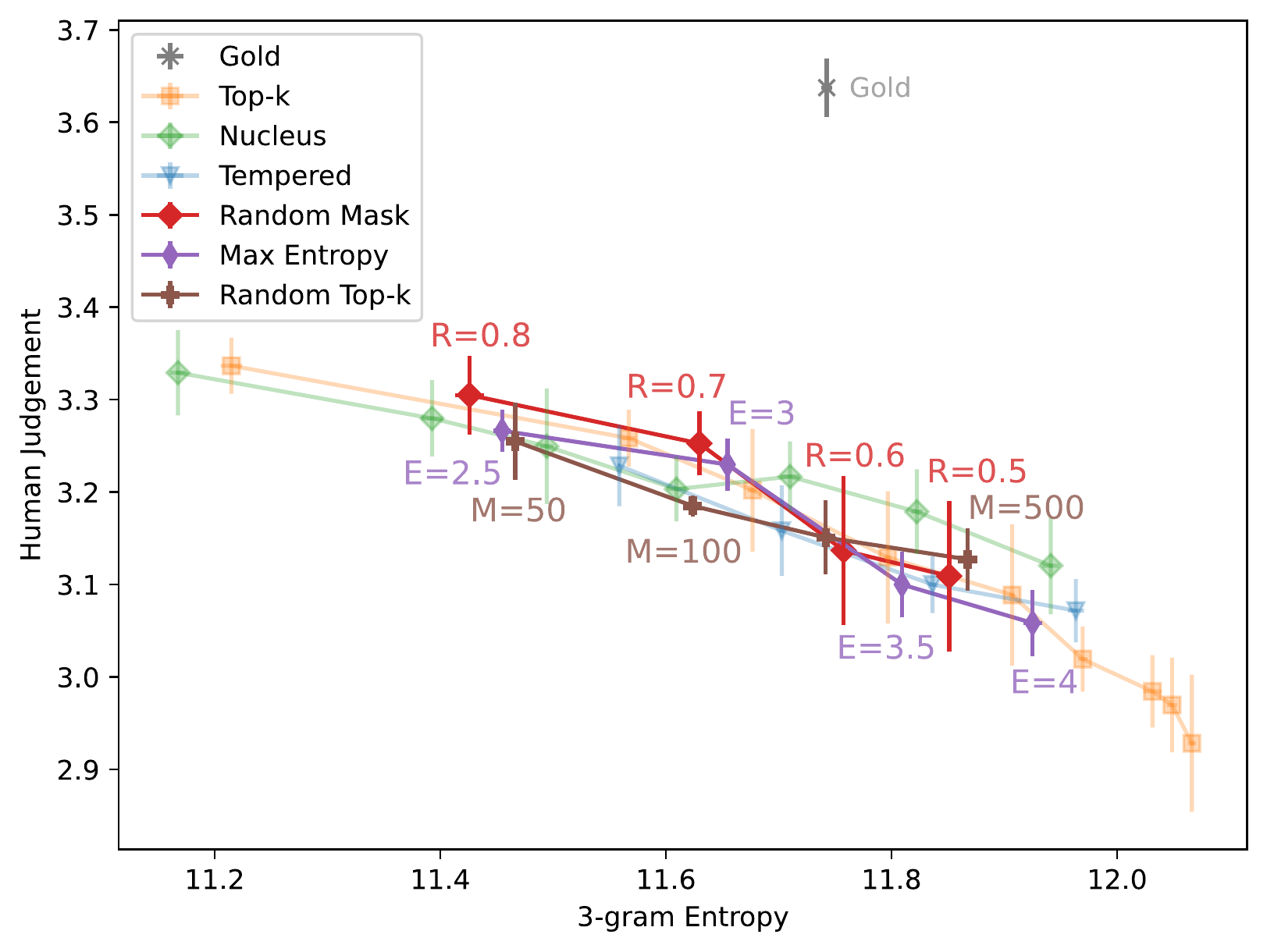}
    \caption{Human evaluation also shows that the proposed sampling algorithms has performance on par with the existing methods on the Gigaword dataset. Appendix \ref{sec:ngram_entropy} repeats this plot with self-BLEU.}
    \label{fig:giga_humaneval_newsampling}
\end{figure}

\definecolor{emphasize}{rgb}{1.0, 0.25, 0.25}
\begin{table*}[]
\centering
\small
\aboverulesep=0pt
\belowrulesep=0pt
\renewcommand{\arraystretch}{1.2}
\begin{tabular}{p{2.1cm}p{13cm}}
\toprule
 \multicolumn{1}{c}{\textbf{Sampling}} &
 \multicolumn{1}{c}{\textbf{Conditional Samples}} \\ \midrule
 \multicolumn{2}{c}{\textbf{Existing Sampling Algorithms}} \\ \hline
\cellcolor{lightblue}\textit{Top-$k$ \newline (K = 30)} & 
    \cellcolor{lightblue}\textit{steven spielberg’s dreamworks movie studio} said monday it was filing a lawsuit, accusing us studio executives of defrauding hundreds of thousands of dollars in refunds and other damages. \\
\textit{Nucleus \newline (P = 0.80)} &
    \textit{steven spielberg's dreamworks movie studio} has failed to attract the kind of business and development investors that jeffrey hutchinson dreamed up in the past. \\
 \cellcolor{lightblue}\textit{Tempered \newline (T = 0.85)} &
  \cellcolor{lightblue}\textit{steven spielberg's dreamworks movie studio} plans to spend the rest of the year producing the high-speed thriller "the earth's path"  and an upcoming sequel, the studio announced on wednesday. \\ \hline
  \multicolumn{2}{c}{\textbf{Property-satisfying Sampling Algorithms}} \\ \hline
  
 \textit{Random Top-$k$ \newline (R = 90)} &
    \textit{steven spielberg’s dreamworks movie studio} is planning to make a movie about a young man who is a $<$unk$>$, a man who has a dream of being the first man to be born with the ability to walk on water. \\
 \cellcolor{lightblue}\textit{Max Entropy \newline (E = 2.75)} &
  \cellcolor{lightblue}\textit{steven spielberg's dreamworks movie studio} has agreed to pay \$ \#.\# million to director john nichols (£ \#.\# million, \#\#\#, a record in the studio circulation ), the studio announced sunday.. \\ \hline
  \multicolumn{2}{c}{\textbf{Property-violating Sampling Algorithms}} \\ \hline
  
\textit{Random Mask \newline (R = 0.75)} &
  \textit{steven spielberg's dreamworks movie studio} scored a big win with a \$ \#\#.\# million ( euro \#\#.\# million ) direct-to-video ( dvds ) deal to develop the \#\#\#\# short story "the rose garden". \\
 \cellcolor{lightblue}\textit{Noised Top-$k$ \newline (K=50, W=5e-3)} &
  \cellcolor{lightblue}\textit{steven spielberg's dreamworks movie studio} is in disarray and has a few directors and a lot of stock involved, leaving it only a matter of time before \textcolor{emphasize}{spielberg's departure from the nobel peace prize}. \\ 
 \textit{Target Entropy \newline (E = 2.75)} &
  \textit{steven spielberg's dreamworks movie studio} production scored an action \textcolor{emphasize}{boost m boom, nabbing an 'd} after the \#\#th instal specialization with nominations of fritz, ika, ivan english ape and evlyn mcready. \\ \bottomrule
\end{tabular}
\caption{Generated sequences with the same prefix \textit{steven spielberg’s dreamworks movie studio} by different sampling algorithms. The hyperparameters are chosen such that the algorithms yield roughly the same diversity measured by self-BLEU. The poor-quality spans are higlighted in \textcolor{emphasize}{red}.}
\label{tab:generated-samples-competitive}
\end{table*}

We now compare the generation performance of the property-satisfying sampling algorithms (designed in Section \ref{sec:sampling_deisnged}) with the existing sampling algorithms. The results from the Gigaword dataset are shown in Figure \ref{fig:giga_sufficiency_bleu} (for automatic evaluation) and Figure \ref{fig:giga_humaneval_newsampling} (for human evaluation). For completeness, we also replicate Figure \ref{fig:giga_humaneval_newsampling} with self-BLEU as the diversity measure in Appendix \ref{sec:ngram_entropy}. We also present results from automatic evaluation on the Wikitext-103 dataset in Figure \ref{fig:wiki_bleu_competitive}, with consistent observations. 

The evaluations consistently show that the performance of random top-$k$ and max entropy sampling (and random mask sampling in last section) is on par with top-$k$, nucleus, and tempered sampling. These results strengthen the importance of the identified properties in that, new sampling algorithms could get competitive generation performance as long as they meet the identified properties.

\begin{figure}[h!]
    \centering
    \includegraphics[width=\columnwidth]{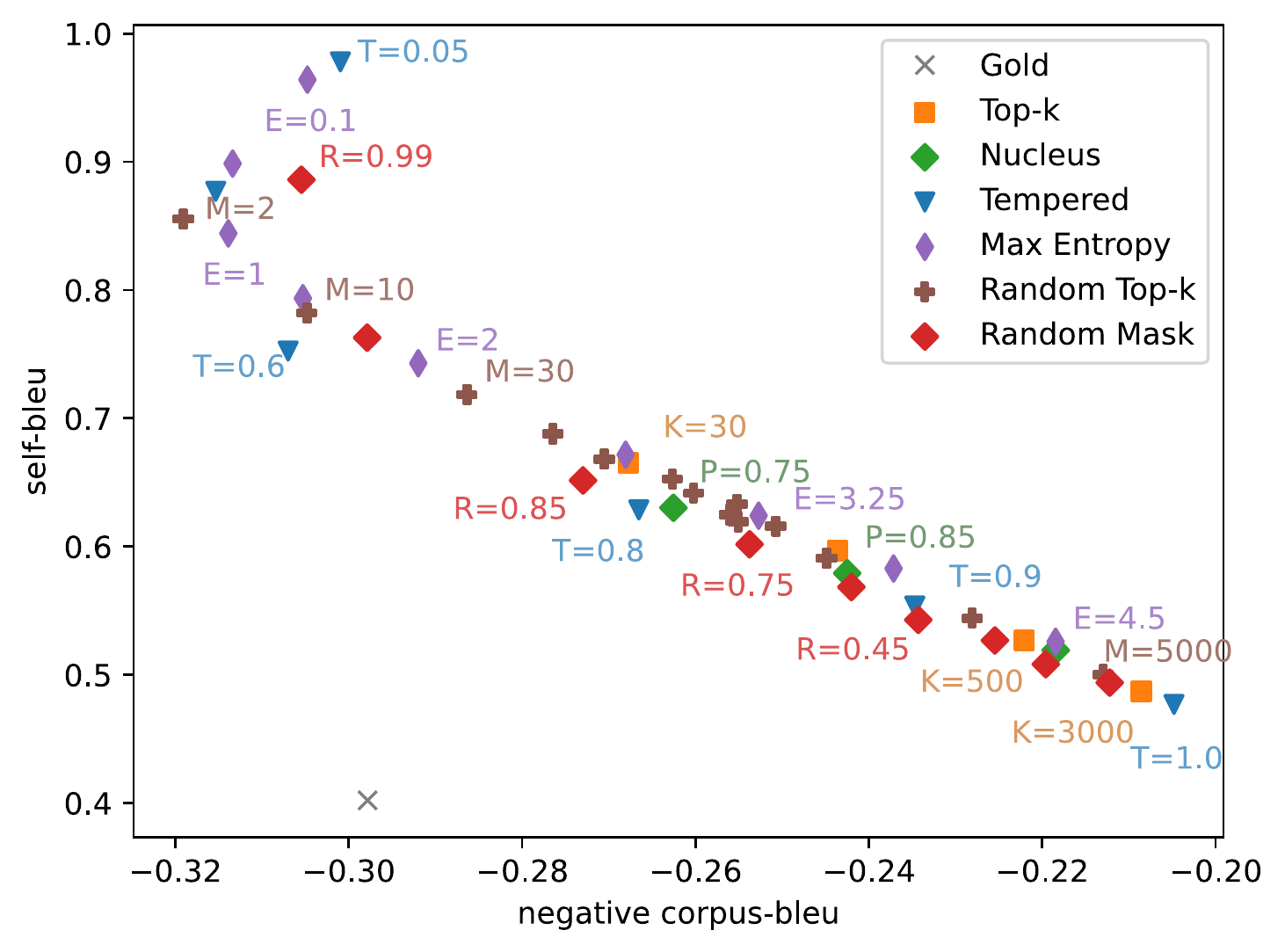}
    \caption{Automatic evaluation on the Wikitext-103 dataset: The performance of proposed sampling algorithms are on par with top-$k$, nucleus, and tempered sampling.}
    \label{fig:wiki_bleu_competitive}
\end{figure}

\subsection{Qualitative Analysis}
We list samples from the proposed sampling algorithms and compare them with the existing ones in Table \ref{tab:generated-samples-competitive}. We choose the hyperparameter of each sampling algorithm so that each algorithm exhibits a similar level of diversity (as measured by self-BLEU). By manual inspection, we find that the quality of samples from property-satisfying sampling algorithms is on par with samples from the existing algorithms. In particular, the samples from random top-$k$, max entropy, and random masked sampling are all coherent and informative. 

In contrast, the samples from noised top-$k$ and target entropy algorithms, tend to be less semantically and syntatically coherent. In particular, the target entropy sampling algorithm, which obtains the lowest quality score measured by corpus-BLEU, lacks basic language structure. In comparison to target entropy, noised top-$k$ is syntatically coherent, but exhibits logical and factual inconsistencies. These observations aligns with the results we get from automatic evaluation.

\section{Related Works}
\label{sec:related}

Despite the popularity of sampling algorithms in natural language generation, a rigorous comparison or scrutiny of existing algorithms is lacking in the literature. \citet{Holtzman2020The} proposes nucleus sampling, and compare it with top-$k$ sampling \citep{fan-etal-2018-hierarchical}. However, only a few hyperparameter configurations are tested. In \citet{huse19tatsunori} and \citet{shortgan18massimo}, temperature sampling is used and the hyperparameter $T$ is tuned to trade-off between diversity and quality, but it lacks comparisons with other sampling algorithms. \citet{welleck2020consistency} studies the \textit{consistency} of existing sampling and decoding algorithms, without comparing the generation performance.

In this work we mainly use the quality-diversity trade-off \citep{shortgan18massimo} to conduct a comparison of different sampling algorithms. Parallel to our work, \citet{zhang2020trading} also uses the quality-diversity trade-off to compare top-$k$, nucleus, and tempered sampling. Their observation is similar to ours: The performance of the existing algorithms are close with no significant gap.

More importantly, the underlying reasons for the success of various sampling algorithms remain poorly understood. \citet{zhang2020trading} proposes the \textit{selective} sampling algorithm, which fails to outperform existing approaches. This failed attempt suggests the need for a better understanding of the strengths and weaknesses of existing methods. To the best of our knowledge, our work provides the first systematic characterization of sampling algorithms, where we attribute the success of existing sampling algorithms to a shared set of properties. We show that we can propose novel sampling algorithms based on the identified properties, and reach competitive generation performance as measured by both automatic and human evaluation.

\section{Limitations and Future Work}
Our core contribution is the three properties of sampling algorithms that we conjecture are crucial for competitive generation performance. While we design a set of experiments to validate their necessity and sufficiency, the observations we make are still  empirical. We emphasize that \textbf{it is completely possible that there exists some crucial property, that is yet to be discovered, and can lead to significantly better generation performance}. Therefore, the exploration of novel sampling algorithms \citep{zhang2020trading} should still be encouraged.

On the other hand, to provide a comprehensive study, we focus on the open-ended language generation task with the GPT-2 model. As future work, it would be interesting to check whether our observations also hold on other tasks such story generation or dialogue response generation, or with weaker language models in low-resource setting.

\section{Conclusion}
This work studies sampling algorithms for the open-ended language generation task. We show that the existing algorithms, namely top-$k$, nucleus, and tempered sampling, have similar generation performance as measured by the quality-diversity trade-off evaluation. Motivated by this result, we identify three key properties that we prove are shared by the existing algorithms. To validate the importance of these identified properties, we design a set of new sampling algorithms, and compare their performance with the existing sampling algorithms. We find that violation of the identified properties may lead to drastic performance degradation. On the other hand, we propose several novel algorithms, namely random top-$k$ and max entropy sampling, that meet the identified properties. We find that their generation performance is on par with the existing algorithms. 

\section*{Acknowledgments}
The authors sincerely thank Yixin Tao, Jingzhao Zhang and Yonatan Belinkov for useful discussions. This work was partly supported by Samsung Advanced Institute of Technology (Next Generation Deep Learning: from pattern recognition to AI), Samsung Electronics (Improving Deep Learning using Latent Structure). Kyunghyun Cho thanks CIFAR, Naver, eBay, NVIDIA and Google for their support.

\bibliography{anthology,emnlp2020}
\bibliographystyle{acl_natbib}

\afterpage{\null\newpage}
\newpage
\clearpage
\appendix

\section{Auxiliary Plots}
\label{app_auxiliaryplots}

We show the importance of preserving the token with the largest probability ($p_1$) in the proposed random mask sampling. For comparison, we relax the constraint and define the \textit{random mask-all} sampling:
\begin{definition}
(\textbf{Random Mask-all})  
The only difference between random mask-all sampling and random mask sampling is that we allow the $p_1$ token to be masked. We formulate it below:
\begin{equation}
    \hat{p}_i = \frac{p'_i}{\sum^{|V|}_{j=1} p'_j},
\end{equation}
where $p'_i=p_i \cdot \mathbbm{1}\{u_i > R \}$ and $u_i \sim U(0,1)$. 
\end{definition}

In Figure \ref{fig:giga_randomspace_whetherpreserve}, we show that if $p_1$ is allowed to be masked, the generation performance will be seriously degraded.
\begin{figure} [h]
    \centering
    \includegraphics[width=\columnwidth]{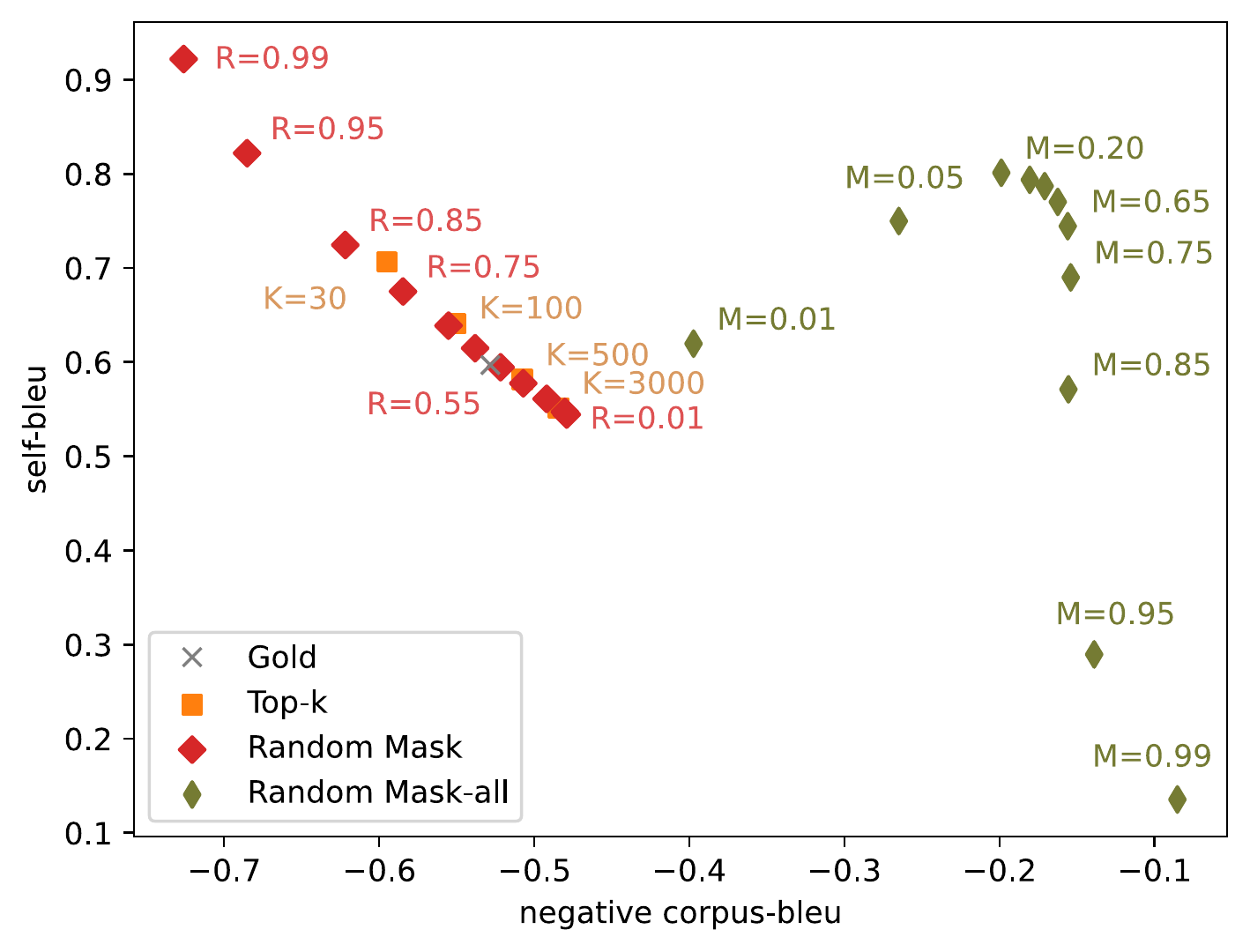}
    \caption{The random mask-all sampling, where $p_1$ is allowed to be masked, is shown to have worse performance than the random mask sampling. The dataset is Giagword.}
    \label{fig:giga_randomspace_whetherpreserve}
\end{figure}

\section{Proof for Proposition \ref{prop_propertyholds}}
\label{app_proof}
In this section we prove Proposition \ref{prop_propertyholds}. 

Firstly, it is straightforward to prove that Property \ref{p_order} (order preservation) holds for the top-$k$, nucleus and tempered sampling and we omit the proof here.

For Property \ref{p_slope} (slope preservation), it holds trivially for nucleus and top-$k$ sampling. We prove it for tempered sampling in the following lemma:
\begin{lemma}
Property \ref{p_slope} holds for tempered sampling (Definition \ref{def_temp}).
\end{lemma}
\begin{proof}
Remember that the tempered sampling with hyperparameter $T$ defines the follow transformation: $\hat{p_i}=\frac{p'_i}{\sum_j p'_j}$, where $p'_i=\exp(\log(p_i)/T)$ . 
We set $Z=\sum_j p'_j$, then $\forall \hat{p}_i > \hat{p}_j > \hat{p}_k > 0$ we have
\begin{equation}
\small
\begin{split}
& \frac{\log\hat{p}_i-\log\hat{p}_j}{\log\hat{p}_j-\log\hat{p}_k} \\
= &\frac{\log p'_i - \log Z -\log p'_j + \log Z}{\log p'_j - \log Z -\log p'_k + \log Z} \\
= &\frac{\log p'_i -\log p'_j}{\log p'_j -\log p'_k} ~ \text{($\log Z$ is cancelled)} \\
= &\frac{\log(p_i)/T -\log(p_j)/T}{\log(p_j)/T -\log(p_k)/T} \\
= & \frac{\log(p_i) -\log(p_j)}{\log(p_j) -\log(p_k)} 
\end{split}
\end{equation}
\end{proof}

Only Property \ref{p_entropy} (entropy reduction) is left. We now prove it holds for top-$k$ / nucleus sampling:
\begin{lemma}
Property \ref{p_entropy} holds for transformations defined by top-$k$ or nucleus sampling (Definition \ref{def_topk} and \ref{def_nucleus}).
\end{lemma}
\begin{proof}
We first consider the change of entropy when the token with the smallest probability ($p_{|V|}$) is removed from the original distribution ($\hat{p}_i=\frac{p_i}{\sum^{|V|-1}_{j=1}p_i}, 1 \leq i < |V|$):
\begin{equation}
\small
\begin{split}
& -\mathcal{H}(\vp) = \sum_{i=1}^V p_i \log p_i  \\
    &= \sum_{i=1}^{V-1} p_i \log p_i + p_{|V|} \log p_{|V|} \\
    &= (1- p_{|V|}) \sum_{i=1}^{V-1} \frac{p_i}{1-p_{|V|}}\log p_i + p_{|V|} \log p_{|V|}
\\
&= \sum_{i=1}^{V-1} \frac{p_i}{1-p_{|V|}}\log \frac{p_i}{1-p_{|V|}} + \underbrace{\log (1-p_{|V|})}_{< 0} \\ & ~~~~ + p_{|V|} \left(\log p_{|V|} - \sum_{i=1}^{V-1} \frac{p_i}{1-p_{|V|}} \log p_i\right) \\
&< \sum_{i=1}^{V-1} \hat{p}_i \log \hat{p}_i + p_{|V|} \left(\log p_{|V|} - \sum_{i=1}^{V-1} \frac{p_i}{1-p_{|V|}} \log \underbrace{p_i}_{> p_{|V|}}\right)
\\
&< \sum_{i=1}^{V-1} \hat{p}_i \log \hat{p}_i + p_{|V|} \left(\log p_{|V|} - \underbrace{\sum_{i=1}^{V-1} \frac{p_i}{1-p_{|V|}} \log p_{|V|}}_{=\log p_{|V|}}\right)
\\
&= \sum_{i=1}^{V-1} \hat{p}_i \log \hat{p}_i = -\mathcal{H}(\hat{\vp}) 
\end{split}
\end{equation}
Therefore, we get $\mathcal{H}(\hat{\vp}) < \mathcal{H}(\vp)$. 

By induction (iteratively removing the last token), it is now easy to see that the top-$k$ or nucleus transformation strictly decrease the entropy of the sampling distribution.
\end{proof}

Finally, we prove Property \ref{p_entropy} (entropy reduction) holds for tempered sampling:

\begin{lemma}
Property \ref{p_entropy} holds for the transformation defined by tempered sampling (Definition \ref{def_temp}).
\end{lemma}
\begin{proof}
For convenience, we first rewrite the Temperature transformation:
\begin{equation}
\hat{p}_i=p_i^\alpha = \frac{\exp(-\alpha e_i)}{\sum_j \exp(-\alpha e_j)}
\end{equation}
where $e_i = -\log (p_i)$ and $\alpha = \frac{1}{T}$. 
The entropy can be written as:
\begin{equation}
\small
    \begin{split}
        \mathcal{H}(\vp^\alpha) &= -\sum_i \frac{\exp(-\alpha e_i)}{\sum_j \exp(-\alpha e_j)} \log \frac{\exp(-\alpha e_i)}{\sum_j \exp(-\alpha e_j)} 
\\
&= \log \sum_j \exp(-\alpha e_j) + \alpha \sum_i e_i \frac{\exp(-\alpha e_i)}{\sum_j \exp(-\alpha e_j)}
    \end{split}
\end{equation}
Next, we take derivative w.r.t $\alpha$:
\begin{equation}
\small
    \begin{split}
        &\frac{\partial \mathcal{H}}{\partial \alpha} =\underbrace{-\sum_i e_i \frac{\exp(-\alpha e_i)}{\sum_j \exp(-\alpha e_j)}
+\sum_i e_i \frac{\exp(-\alpha e_i)}{\sum_j \exp(-\alpha e_j)}}_{=0}\\ 
&+\alpha \frac{\partial}{\partial \alpha} \sum_i e_i \frac{\exp(-\alpha e_i)}{\sum_j \exp(-\alpha e_j)}
\\
&= 
\alpha \sum_i e_i \underbrace{\left[\frac{\partial}{\partial \alpha} \log \frac{\exp(-\alpha e_i)}{\sum_j \exp(-\alpha e_j)}\right] \left[\frac{\exp(-\alpha e_i)}{\sum_j \exp(-\alpha e_j)}\right]}_{\text{log-derivative trick}}
\\
&= \alpha \sum_i e_i\left[ -e_i + \sum_{j'} e_{j'} \frac{\exp(-\alpha e_i)}{\sum_j \exp(-\alpha e_j)} \right] \\ 
&~~~~\left[\frac{\exp(-\alpha e_i)}{\sum_j \exp(-\alpha e_j)} \right] \\
&= -\alpha \mathbb{E}_{p^\alpha} \left[ e_i^2 - e_i \mathbb{E}_{p^\alpha}[e_i] \right]
\\
&= -\underbrace{\alpha}_{>0} \underbrace{\left( \mathbb{E}_{p^\alpha}[e_i^2] - \mathbb{E}_{p^\alpha}[e_i]^2 \right)}_{=\text{Var}_{p^\alpha}[e_i] \geq 0}
\\
&< 0
    \end{split}
\end{equation}
We can now easily get $\frac{\partial \mathcal{H}}{\partial T} = \frac{\partial \mathcal{H}}{\partial \alpha} \frac{\partial \alpha}{\partial T} > 0$. 
Therefore, when we apply a tempered transformation with $T<1$, the entropy will strictly decrease comaparing to the original distribution (where $T=1$).
\end{proof}

\section{Mechanical Turk Setup}
\label{appendix:mturk_details}
Our crowdworkers were required to have a HIT acceptance rate higher than 95\%, and be located in the United States. In total, 602 crowdworkers completed our tasks. In order to ensure that we had quality data, we filtered the crowdworker annotations for workers that spent at least 45 seconds on the aggregate task (or 4.5 seconds rating each sentence). 51 crowdworkers were filtered out through this process. Screenshots of our instructions and task are available in Figure(s) \ref{fig:mturk_task} and \ref{fig:mturk_example} respectively.

\begin{figure}[h!]
    \centering
    \includegraphics[width=\columnwidth]{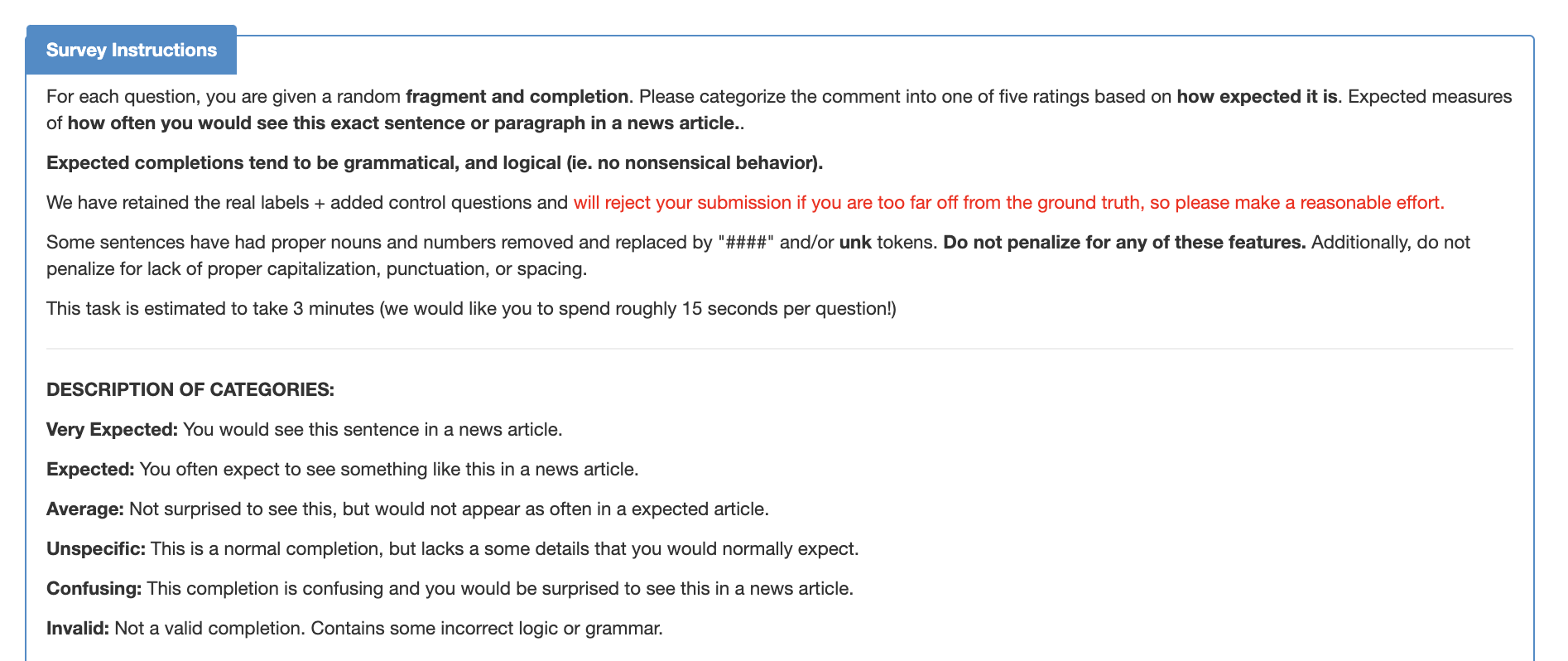}
    \caption{Our instructions for crowdworker task.}
    \label{fig:mturk_task}
\end{figure}

\begin{figure}[h!]
    \centering
    \includegraphics[width=\columnwidth]{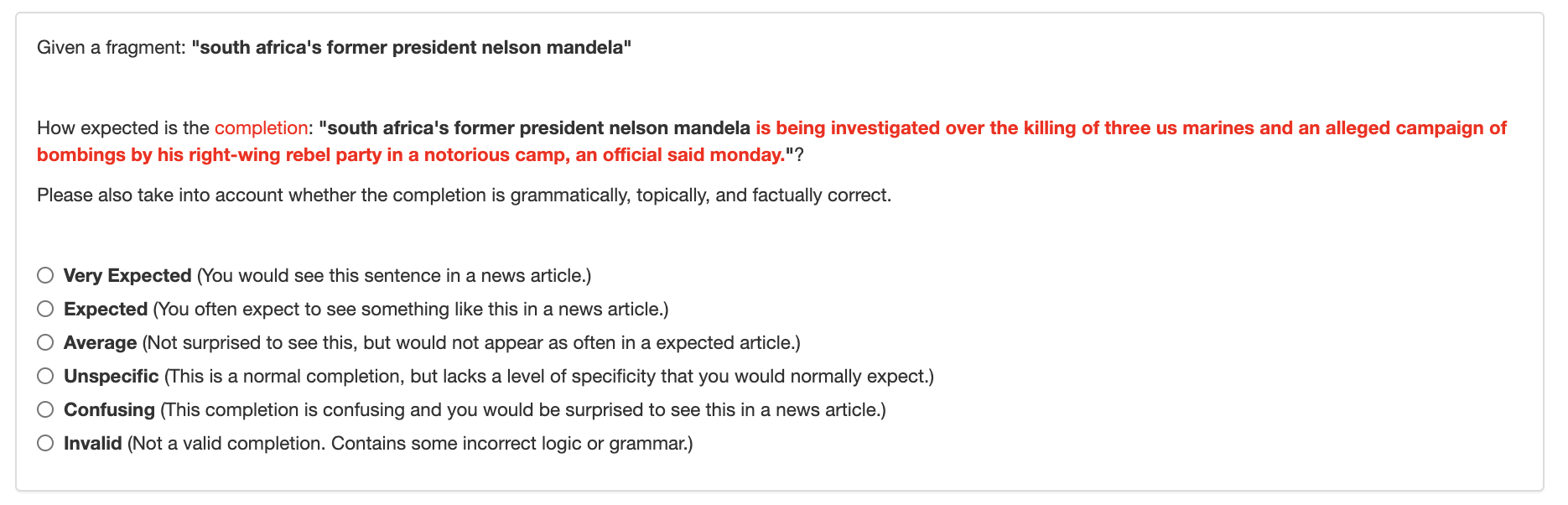}
    \caption{An example of the task given to crowdworkers.}
    \label{fig:mturk_example}
\end{figure}

\section{Convergence of Human Evaluation}
\label{appendix:convergence}
When we conduct human evaluation, we provide crowdworkers with 200 generated samples for some configuration, and ask 25 different crowdworkers to evaluate the same sample. However, a reasonable question is whether our human evaluations are converging to some underlying true rating, or whether we need more samples or replicas. %

Figure \ref{fig:samples_sweep} and \ref{fig:replica_sweep} show that the average scores have roughly converged around 150 samples per configuration, or around 15 replicas per sample. The two figures demonstrate this for nucleus sampling, and this holds true for human evaluations of all sampling algorithms. 

\begin{figure}[h!]
    \centering
    \includegraphics[width=0.9\columnwidth]{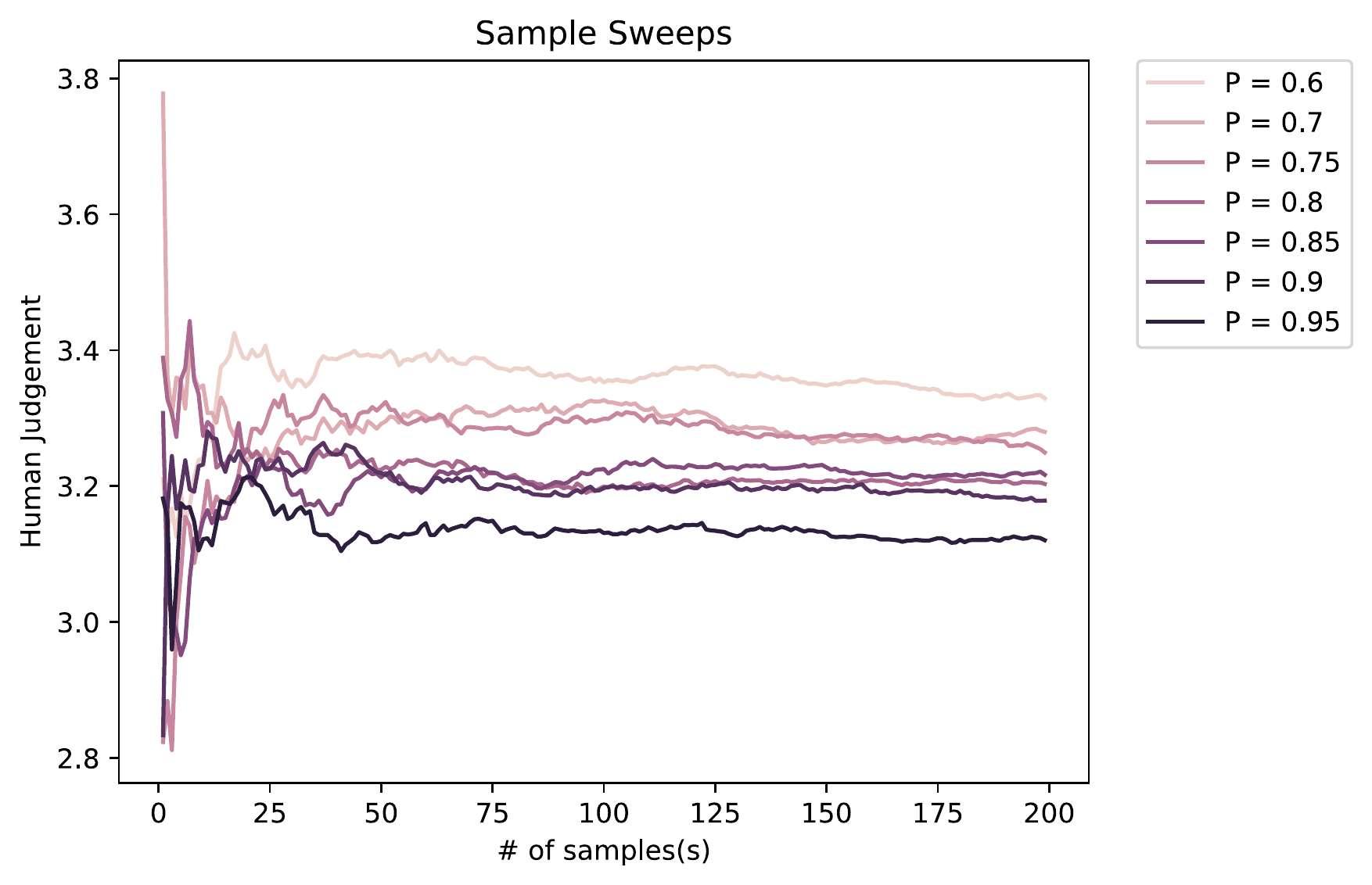}
    \caption{We see that we obtain a reasonable estimate of sample quality around 150 samples per configuration.}
    \label{fig:samples_sweep}
\end{figure}

\begin{figure}[h!]
    \centering
    \includegraphics[width=0.9\columnwidth]{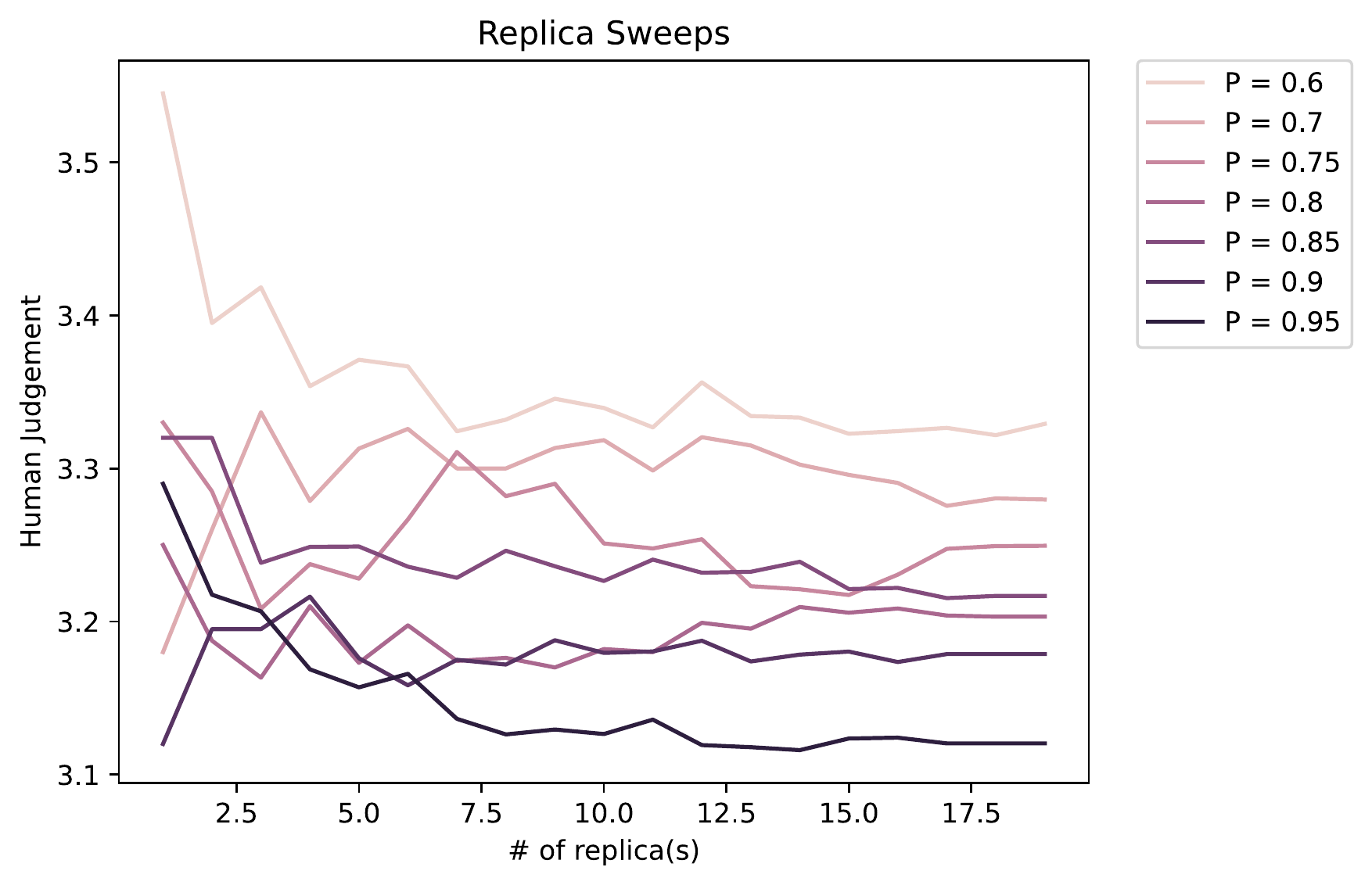}
    \caption{We see that we obtain a reasonable estimate of sample quality with around 15 ratings per sample.}
    \label{fig:replica_sweep}
\end{figure}

\section{Additional Model-Generated Samples}
\begin{table*}[]
\centering
\small
\aboverulesep=0pt
\belowrulesep=0pt
\renewcommand{\arraystretch}{1.2}
\begin{tabular}{p{2.1cm}p{13cm}}
\toprule
\multicolumn{1}{c}{\textbf{Sampling}} &
  \multicolumn{1}{c}{\textbf{Conditional Samples}} \\ \midrule
  \multicolumn{2}{c}{\textbf{Existing Sampling Algorithms}} \\ \hline
\cellcolor{lightblue}\textit{Top-K \newline (K = 15)} & 
    \cellcolor{lightblue}\textit{as the rest of his denver broncos teammates} prepared for the game against denver, jay kasey could not help but think of his teammates and friends who worked hard in preparation for that night's game. \\
\textit{Nucleus \newline (P = 0.65)} &
    \textit{as the rest of his denver broncos teammates} slumped and buried themselves in their work, broncos quarterback leon johnson moved to the locker room monday and called his parents. \\
 \cellcolor{lightblue}\textit{Temperature \newline (T = 0.7)} &
  \cellcolor{lightblue}\textit{as the rest of his denver broncos teammates} gathered in an auditorium to watch more stretching drills, ben holtz gave an emotional speech : we're running out of time to win a championship ring. \\ \hline
  \multicolumn{2}{c}{\textbf{Property-satisfying Sampling Algorithms}} \\ \hline
  
 \textit{Random Top-K \newline (R = 30)} &
    \textit{as the rest of his denver broncos teammates} battled through their own stretch of the nfl playoffs, the quarterback began throwing the ball in the fourth quarter. \\
 \cellcolor{lightblue}\textit{Max Entropy \newline (E = 2.75)} &
  \cellcolor{lightblue}\textit{steven spielberg's dreamworks movie studio} has agreed to pay \$ \#.\# million to director john nichols (£ \#.\# million, \#\#\#, a record in the studio circulation ), the studio announced sunday.. \\ \hline
  \multicolumn{2}{c}{\textbf{Property-violating Sampling Algorithms}} \\ \hline
  
\textit{Random Mask \newline (R = 0.75)} &
  \textit{as the rest of his denver broncos teammates} connect with a player that the team didn't expect to become a starter, quarterback james crosby speaks out about colin peterson's passion for the game. \\
 \cellcolor{lightblue}\textit{Noised Top-K \newline (K=20, W=5e-3)} &
  \cellcolor{lightblue}\textit{as the rest of his denver broncos teammates} start making room for nerdy bundles or twiggy pitchers, coach william perez might have to cut a big, bold note cut ready to \textcolor{emphasize}{console wife join them in iraq}. \\ 
 \textit{Target Entropy \newline (E = 2.5)} &
  \textit{as the rest of his denver broncos teammates} scratched out their locker rooms,  \textcolor{emphasize}{cleanDeath Yo Communities wander edge extingustretched cords429 Mohnegie wildfires}. \\ \bottomrule
\end{tabular}
\caption{The samples conditioned on \textit{as the rest of his denver broncos teammates}, and the hyperparameters for a given sampling algorithm. The poor quality spans are higlighted in \textcolor{emphasize}{red}.}
\label{tab:generated-samples-extra}
\end{table*}

Table \ref{tab:generated-samples-extra} shows some additional samples from each of the sampling algorithms described in the paper. Similarly, we have chosen hyperparameters for each sampling method that yields a similar diversity (measured by self-BLEU) to the top-$k$ configuration where $K=15$. We observe that all sampling algorithms except for noised top-$k$ and target entropy, yield similar quality samples. For noised top-$k$ and target entropy, we see that these samples tend to degenerate towards the end of the sentence, indicating violation of the identified properties may possibly lead towards degraded performance.

\section{Human Evaluation with Self-BLEU as Diversity Metric}
\label{sec:ngram_entropy}
\begin{figure}[h!]
    \centering
    \includegraphics[width=0.95\columnwidth]{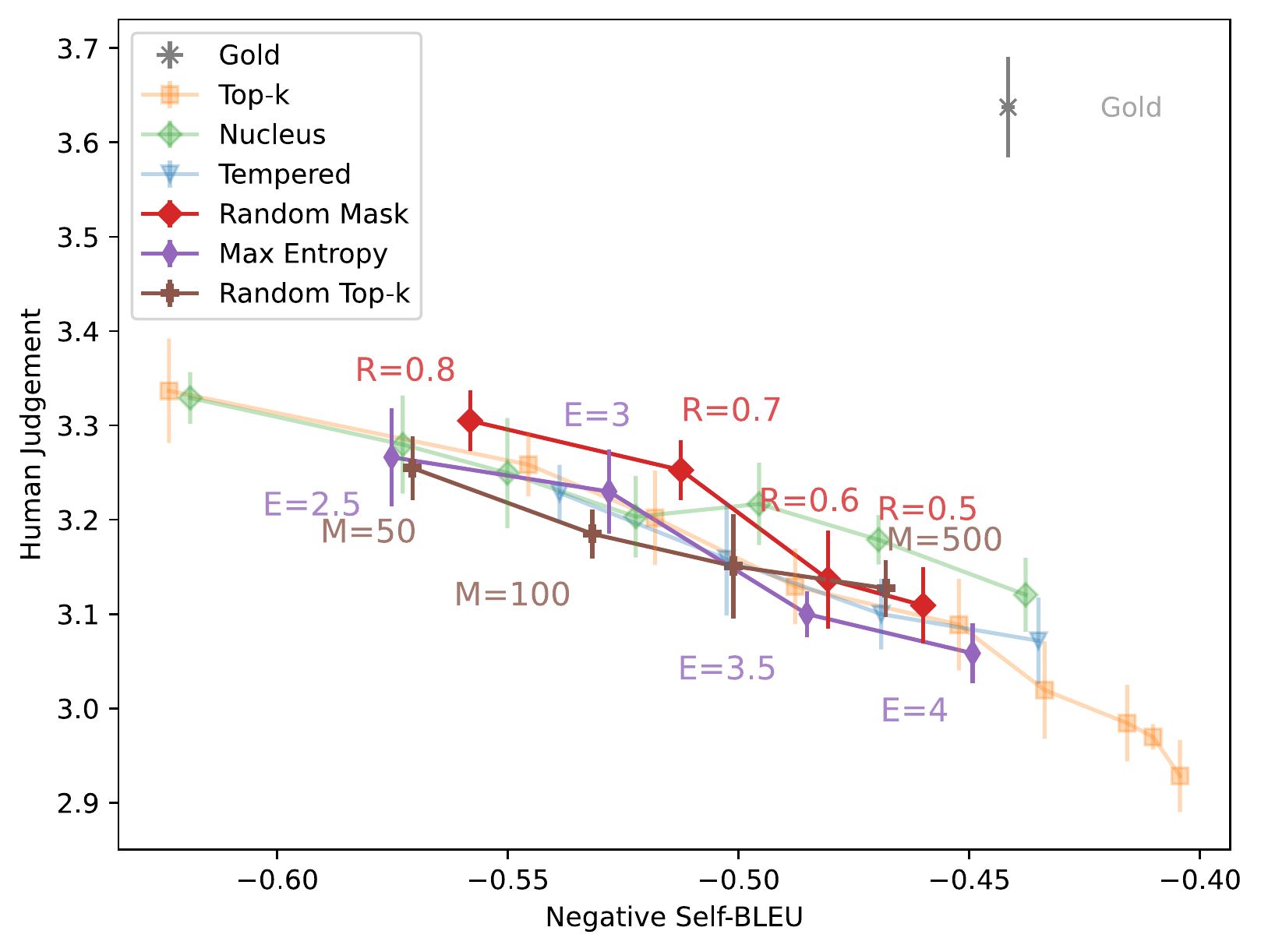}
    \caption{Using self-BLEU as a diversity metric provides similar conclusions as to using n-gram entropy.}
    \label{fig:giga_sbleu}
\end{figure}

Figures \ref{fig:giga_humaneval_existing} and \ref{fig:giga_humaneval_newsampling} measures diversity in terms of 3-gram entropy, while the rest of our work measures diversity in terms of self-BLEU. For completeness, we provide Figure \ref{fig:giga_sbleu} where self-BLEU is used for diversity metric. This figure demonstrates that similar trends can be observed using either 3-gram entropy or self-BLEU.

\end{document}